\def\year{2021}\relax
\documentclass[letterpaper]{article} 
\usepackage{aaai21}  
\usepackage{times}  
\usepackage{helvet} 
\usepackage{courier}  
\usepackage[hyphens]{url}  
\usepackage{graphicx} 
\usepackage{natbib}  
\usepackage{caption} 
\usepackage{subfigure}
\usepackage{multirow}
\usepackage{amssymb}
\usepackage{amsmath}
\usepackage{amsthm}
\usepackage{booktabs}
\newcommand{\tabincell}[2]{\begin{tabular}{@{}#1@{}}#2\end{tabular}}
\newtheorem{thm}{Theorem}

\title{Learning with Group Noise}
\author{
Qizhou Wang\textsuperscript{\rm 1,2,}\thanks{Equal contribution.}, 
Jiangchao Yao\textsuperscript{\rm 3,}\footnotemark[1],
Chen Gong\textsuperscript{\rm 2,4,}\thanks{Corresponding authors. Emails: chen.gong@njust.edu.cn, bhanml@comp.hkbu.edu.hk. }, \\
Tongliang Liu\textsuperscript{\rm 5}, 
Mingming Gong\textsuperscript{\rm 6},
Hongxia Yang\textsuperscript{\rm 3},
Bo Han \textsuperscript{\rm 1,}\footnotemark[2] \\
}

\affiliations{
	\textsuperscript{\rm 1} Department of Computer Science, Hong Kong Baptist University \\
	\textsuperscript{\rm 2} Key Laboratory of Intelligent Perception and Systems for High-Dimensional Information of MoE, \\ School of Computer Science and Engineering, Nanjing University of Science and Technology\\
	\textsuperscript{\rm 3} Data Analytics and Intelligence Lab, Alibaba Group\\
	\textsuperscript{\rm 4} Department of Computing, Hong Kong Polytechnic University \\
	\textsuperscript{\rm 5}  Trustworthy Machine Learning Lab, School of Computer Science, Faculty of Engineering, The University of Sydney\\
	\textsuperscript{\rm 6} School of Mathematics and Statistics, The University of Melbourne \\
}

\begin{document}
\maketitle

\begin{abstract}
Machine learning in the context of noise is a challenging but practical setting to plenty of real-world applications. Most of the previous approaches in this area focus on the pairwise relation (casual or correlational relationship) with noise, such as learning with noisy labels. However, the group noise, which is parasitic on the coarse-grained accurate relation with the fine-grained uncertainty, is also universal and has not been well investigated. The challenge under this setting is how to discover true pairwise connections concealed by the group relation with its fine-grained noise. To overcome this issue, we propose a novel Max-Matching method for learning with group noise. Specifically, it utilizes a matching mechanism to evaluate the relation confidence of each object (\textit{cf.} Figure~\ref{fig:motivation}) w.r.t. the target, meanwhile considering the Non-IID characteristics among objects in the group. Only the most confident object is considered to learn the model, so that the ﬁne-grained noise is mostly dropped. The performance on a range of real-world datasets in the area of several learning paradigms demonstrates the effectiveness of Max-Matching.
\end{abstract}

\vspace{-8pt}
\section{Introduction}
The success of machine learning is closely related to the availability of data with accurate relation descriptions. However, the data quality usually cannot be guaranteed in many real-world applications, \textit{e.g.}, image classification~\cite{li2017webvision}, machine translation~\cite{belinkov2017synthetic}, and object recognition~\cite{yang2020distilling}. To overcome this issue, learning from cheap but noisy assignments has attracted intensive attention. Especially, in the recent years, lots of works have contributed to learning with label noise~\cite{xia2020parts,han2020sigua,chen2020robustness}.

Nevertheless, most of the previous works focus on the pairwise relation with noise as characterized in Figure~\ref{fig:motivation}(a). For notion simplicity, we call it \emph{pairwise noise}. Another type of noise, which is implicitly parasitic on the weak relations as illustrated in Figure~\ref{fig:motivation}(b)-(d), is also general but has not been well investigated. We specially term it \emph{group noise} based on the two following characteristics: 1) it occurs in the group whose coarse-grained relation to the target is correct, while the fine-grained relation of each object in the group to the target might be inaccurate; 2) it is not proper to independently consider fine-grained noisy relations like Figure~\ref{fig:motivation}(a), since objects in one group exhibit  strong Non-IID characteristics. Correlation analysis for each group of objects can help us discover  better evidences to this type of noise. In the following, we enumerate some examples about group noise.




\begin{figure}
\includegraphics[width=3in]{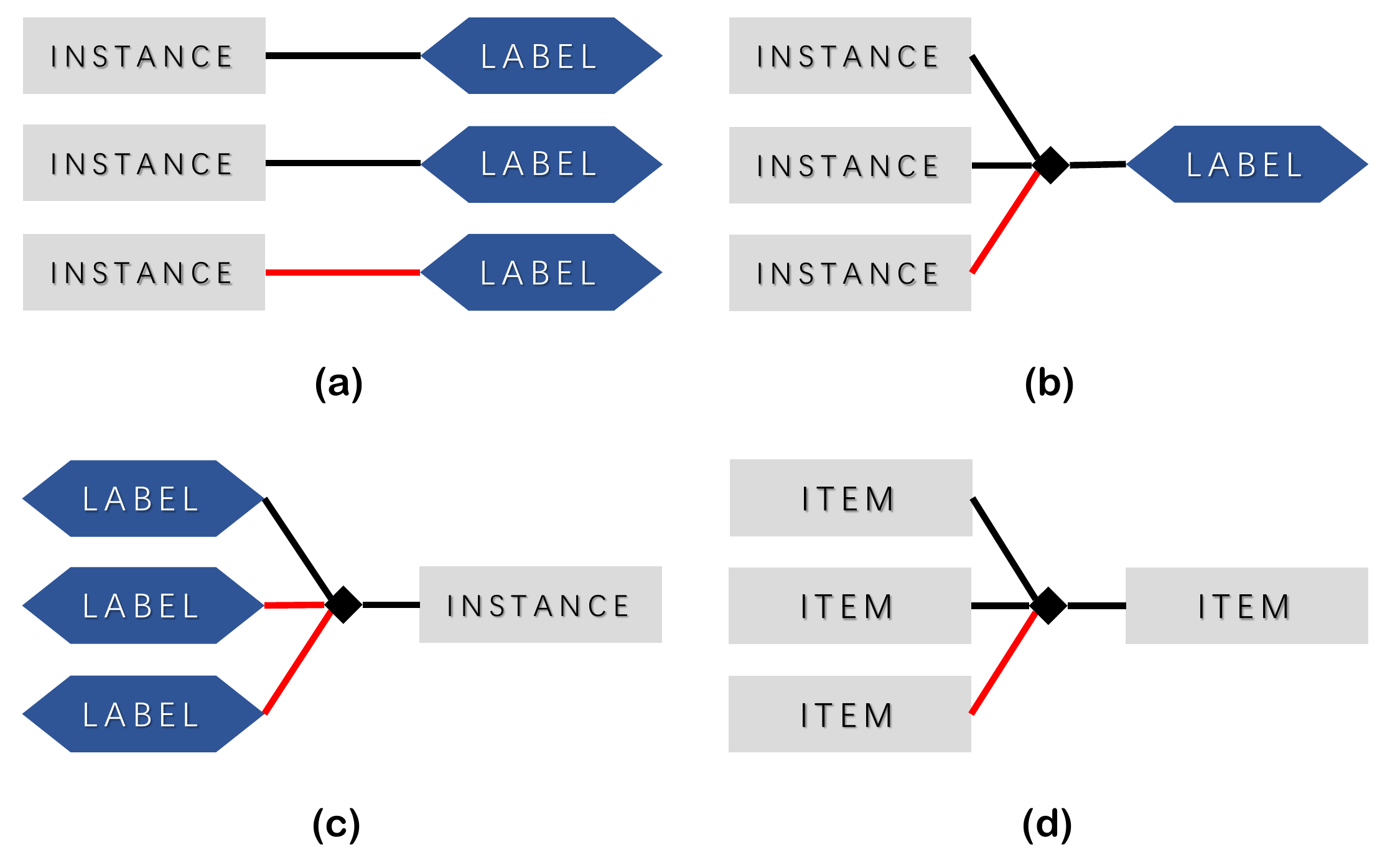}
\centering
\caption{Illustration of the supervised learning with pairwise noise (a) and three settings with group noise (b)-(d), in which the objects are realized by instances, labels, and items respectively. In the figure, black lines represent the correct relations, while red lines mean the incorrect relations.}
\label{fig:motivation}
\end{figure}

\begin{itemize}
    \item Figure~\ref{fig:motivation}(b): In region-proposal-based object localization, each group is a set of regions of one image. Given the image-level label, we aim to find its corresponding regions and remove the irrelevant and background parts. Here, a group of instances (or regions) are weakly supervised by a category label, while some instances are mismatched to this category. This is like the multiple-instance learning problem in the case of the instance-level classification~\cite{liu2012key}, not the more popular bag-level classification~\cite{maron1998framework}. 
   
    \item Figure~\ref{fig:motivation}(c): In face naming, characters' faces may appear simultaneously in a screenshot from the TV serials, and each face is assigned with a set of candidate names in the script or in dialogue. Under this setting, only one name (or label) in the group is correct to the face and all the candidates are correlated due to relationship among characters. From these data, we are to determine the true name of each face, which has also been viewed as a partial-label learning problem~\cite{gong2017regularization}. 
    
    \item Figure~\ref{fig:motivation}(d): In recommender system, item-based collaborative filtering~\cite{sarwar2001item} is a classical method. It builds upon that  the co-occurrence information of item pairs is relatively reliable. However, due to the uncertainty of user behavior in e-commerce, it exists that the historical items are irrelevant to the subsequent items. This introduces the group noise when we consider the sequence of each user as a group for the next-item, leading to the deterioration of applying the NeuralCF model with the fine-grained pairwise relations~\cite{he2017neural}.
\end{itemize}

Although several works more or less explore this type of scenarios, they are usually tailored to their ultimate goals and may distort the characteristics of group noise. For example, previous multiple-instance learning, which considers the instance-level modeling~\cite{settles2008multiple,pao2008based}, may make the strong IID assumption about the instances in the group. Partial-label learning methods~\cite{zhang2016partial} suppose the equal confidence of the candidate labels or model the ground-truth as a latent variable, which might not be very effective. 
Besides, all these works do not explicitly construct the denoising mechanism to avoid the influence of group noise.
 
In this paper, we investigate the problem of learning with group noise, and introduce a novel Max-Matching approach. Specifically, it consists of two parts, a matching mechanism and a selection procedure. The matching mechanism leverages the \emph{pair matching} to evaluate the confidence of relation between the object and the target, meanwhile adopts the \emph{group weighting} to further consider the Non-IID property of objects in the group. The final matching scores are achieved by combining the pair matching and the group weighting, of which the results evaluate both each 
fine-grained relation pair and the object importance in the group. Then, the selection procedure chooses the most confident relation pair to train the model, which at utmost avoids the influence of the irrelevant and mismatched relations. The whole model is end-to-end trained and Figure~\ref{fig:mm} illustrates the structure of Max-Matching. We conduct a range of experiments, and the results indicate that the proposed method can achieve superior performance over baselines from three different learning paradigms with group noise in Figure~\ref{fig:motivation}. 

\vspace{-5pt}
\section{Related Works} 
\subsection{Learning with Pairwise Noise}
For learning with pairwise noise, researchers mainly focus on instances with error-prone labels~\cite{frenay2013classification,algan2019image}, where the noise occurs in pairwise relations between individual instances to their assigned labels. By making assumptions on label assignment, robust loss functions~\cite{manwani2013noise,ghosh2017robust,han2018co,can} and various consistent learning algorithms are proposed~\cite{liu2015classification,han2018masking,xia2019anchor,ags,yao2020dual}.  

\vspace{-5pt}
\subsection{Learning with Group Noise}

For learning with group noise, we have a group of objects collectively connected to the target with the coarse-grained guarantees but the fine-grained uncertainty. Several previous methods, in Multiple-Instance Learning (MIL), Partial-Label Learning (PLL), and Recommender System (RS), have mediately investigated this problem. 

\textbf{MIL} probably is one of the most illustrative paradigms about group noise, of which the supervision is provided for a bag of instances. In MIL, prediction can either be made for bags or individuals, respectively termed as the bag-level prediction and instance-level prediction. For bag-level prediction, many works estimate instance labels as an intermediate step~\cite{ray2005supervised,settles2008multiple,wang2018revisiting}.

However, as suggested by~\cite{vanwinckelen2016instance}, the MIL methods designed for bag classification are not optimal for the instance-level tasks. The methods for instance-level prediction are only studied in the minority but close to the problem of our paper. Existing methods are devised based on key instance detection~\cite{liu2012key}, label propagation~\cite{kotzias2015group}, or unbiased estimation~\cite{peng2019address} with  IID assumptions.

\textbf{PLL}  also relates to the problem of learning with group noise, where each instance is assigned with a group of noisy labels, and only one of them is correct. To avoid the influence of the group noise, two general methodologies, namely, the average-based strategy and the detection-based approach, are proposed. 
The average-based strategy usually treats candidate labels equally, and then adapts PLL to the general supervision techniques~\cite{hullermeier2006learning,cour2011learning,wu2018towards}. The detection-based methods aim at revealing the true label among the candidates, mainly through label confidence learning~\cite{zhang2016partial}, maximum margin~\cite{yu2016maximum}, or alternating optimization~\cite{zhang2015solving,feng2019partial,yao2020deep}. Above methods do not explicitly build the denoising mechanism, which might not be effective in learning with group noise. 

\textbf{RS} targets to recommend the points of interest for users given their historical behaviors. In e-commerce, item-based collaborative filtering~\cite{sarwar2001item,linden2003amazon} has been used as a popular technique, which discovers new items based on the similar ones. It builds upon that the item relation is relatively reliable, so that the unseen true correlations between items can be learned via matrix factorization~\cite{mnih2008probabilistic}, auto-decoders~\cite{sedhain2015autorec}, or deep models~\cite{huang2013learning,xue2017deep,he2017neural,cui2018variational}. Unfortunately, in practice, it is not very easy to accurately construct the such pairwise relation for training, especially in the interest-varying user click sequences. Although more advanced studies mine the multiple interests of users and sequential behavior analysis~\cite{hidasi2018recurrent,wu2019session} to acquire benefits, the effect of group noise has not been well studied yet. Our experiments reveal that eliminating the group noise from the user click sequences for the next-item can effectively improve the performance.

\subsection{Preliminary}

Assume that we have a source set $\mathcal{X}$ and a target set $\mathcal{Y}$. For example, in classification tasks, $\mathcal{X}$ and $\mathcal{Y}$ can be considered as the sample set and the label set respectively. Ideally, we have the collection $S=\left\{(x_i,y_i)\right\}_{i=1}^{n}$ ($n$ is the sample size) for training, where the source object $x_i\in\mathcal{X}$ connects to the target $y_i\in\mathcal{Y}$ via the true pairwise relation. For generality,  we use $f:\mathcal{X}\rightarrow\mathbb{R}^d$ and $g:\mathcal{Y}\rightarrow\mathbb{R}^d$ to map both the objects in $\mathcal{X}$ and $\mathcal{Y}$ into the embedding space. Then, the solution is formulated as the following problem:
\begin{equation}
    f^*,g^*\leftarrow\mathop{\arg\min}_{f,g} \sum\limits_{i=1}^{n} \ell(f(x_i),g(y_i)),
    \label{eq:metric_pair}
\end{equation}
where $\ell:\mathbb{R}^d\times\mathbb{R}^d\rightarrow\mathbb{R}^{+}$ is a proper loss function. After training, the optimal mapping functions are used to make various prediction tasks, such as classification or retrieval. 

However, in many real-world situations, group-level data acquisition is cheaper, in which a group of source objects are collectively connected to a target. Unfortunately, as shown in Figure~\ref{fig:motivation}, some objects in the group can be irrelevant to the target regarding the pairwise relations. This forms the problem of learning with group noise, where we have to handle the noise that is parasitic on $S_\text{group}=\left\{(\bar{X}_i,y_i)\right\}_{i=1}^{n}$. 
Here, 
$\bar{X}_i=\{\bar{x}_{i1},\ldots,\bar{x}_{iK}\}\in\mathcal{X}^K$
contains a set of source objects collectively related to a target object $y_i\in\mathcal{Y}$. Note that $\bar{x}_{ik}$ is different from $x_{ik}$ regarding the notation, indicating there may exist $\bar{x}_{ik} \in \bar{X}_i$, such that $(\bar{x}_{ik},y_i)\notin S$, \textit{i.e.}, $\bar{x}_{ik}$ is mismatched to the target $y_i$ in terms of the pairwise relation. In this setting, we aim at devising a novel objective function $\ell_\text{group}:\mathbb{R}^{d\times K}\times\mathbb{R}^d\rightarrow\mathbb{R}^{+}$ such that
\begin{equation}
     f^*,g^*\leftarrow\mathop{\arg\min}_{f,g} \sum\limits_{i=1}^{n} \ell_\text{group}(F(\bar{X}_i),g(y_i))
     \label{eq:metric_group}
\end{equation}
can find the same optimal mapping functions $f^*,g^*$ as in Eq.~\eqref{eq:metric_pair}, where $F(\bar{X}_i)=\{f(\bar{x}_{i1}),\ldots,f(\bar{x}_{iK})\}$ denotes the set of embedding features. After training, the evaluation is still implemented on the individual pairwise connections between the source object and the target object.

\subsection{Max-Matching}
We need to modify the original loss functions such that the classifier learned with group noise can converge to the optimal one learned without any noise. 

In this section, we introduce a novel method, namely, Max-Matching, for learning with group noise. It consists of two parts, a matching mechanism and a selection procedure. The matching mechanism jointly considers the following two aspects of relation: 
1) the pairwise relation of the source objects to the target; 
2) the relation among the source objects in the group. 
Accordingly, the correctness of the pairwise relations as well as object correlations in the group are revealed by each matching score. Subsequently, based on the results given by the matching mechanism, the selection procedure chooses the best matched object to optimize the model. The group noise can mostly be removed, since the selected object is at utmost guaranteed to be correct regarding its pairwise relation to the target object, and other less 
confident objects in the group are not considered. Formally, the objective function $\ell_{\text{group}}$ of Max-Matching is, 
\begin{align}
    -\max\limits_{\bar{x}_{ik}\in \bar{X}_i} \{\log\underbrace{ \hat{P}\left(y_i|\bar{x}_{ik};f,g\right)}_{\text{Pair Matching}}+ \log\underbrace{ \hat{P}\left(\bar{x}_{ik}|\bar{X}_i;f,g\right)}_{\text{Group Weighting}}\}, \label{eq:loss_max_matching}
\end{align}
where $\hat{P}(\cdot)$ denotes the estimated probability. Note that, the two terms in Eq.~\eqref{eq:loss_max_matching} are equally combined\footnote{Non-equal combination with proper tradeoff may lead to better performance, which is left for future exploration in our work.} 
and they are interdependent in the training phase. 
The second term helps the reliable pairwise relation to be identified, and the first term also boosts the weighting measure to be learned. 

In the following, we explain the intuition behind Eq.~\eqref{eq:loss_max_matching}. In learning with group noise, we have no explicitly clean pairwise relations that can be directly used for training. Therefore, we inevitably build a weighting schema to measure the importance of the data in the group, which we assume is $\hat{P}(\bar{x}_{ik}|\bar{X}_i;f,g)$. Then, following the law of the total probability, it might be possible to decompose $\hat{P}(y_i|\bar{X}_i;f,g)$ into a probabilistic term \textit{w.r.t.} the target for 
each object $\bar{x}_{ik}\in\bar{X}_i$ combined with $\hat{P}(\bar{x}_{ik}|\bar{X}_i;f,g)$. However, the optimization obstacle caused by the integral will prohibit this choice.
In this case, Eq.~\eqref{eq:loss_max_matching} is an alternative approximation to this goal, which we use the following theorem to formulate.
\begin{thm}
Assume $\bar{X}_i=\{\bar{x}_{i1},\ldots,\bar{x}_{iK}\}\in\mathcal{X}^K$ collectively connects to the target $y_i$, where there is at least one true pairwise relation $(\bar{x}_{ik},y_i)$ and some possible pairwise relation noise. Then, optimizing Eq.~\eqref{eq:loss_max_matching} is approximately optimizing all pairwise relations with weights to learn the optimal mapping functions $f^*$ and $g^*$.
\end{thm}
\begin{proof}
According to the law of total probability, the log-likelihood on the coarse-grained relation $(\bar{X}_i, y_i)$ has a following decomposition and the lower-bound approximation,
\begin{align}
    & \log{\sum_{\bar{x}_{ik}\in \bar{X}_i}\hat{P}(y_i|\bar{x}_{ik};f,g)\hat{P}(\bar{x}_{ik}|\bar{X}_i;f,g)}  \nonumber \\
    & \geq \log{\max\limits_{\bar{x}_{ik}\in \bar{X}_i}\left\{\hat{P}(y_i|\bar{x}_{ik};f,g)\hat{P}(\bar{x}_{ik}|\bar{X}_i;f,g)\right\}}  \\
    & = \max\limits_{\bar{x}_{ik}\in \bar{X}_i}\left\{\log \hat{P}(y_i|\bar{x}_{ik};f,g) + \log\hat{P}(\bar{x}_{ik}|\bar{X}_i;f,g) \right\} \nonumber
\end{align}
The first line in the above deduction can be considered as a weighted counterpart of Eq~\eqref{eq:metric_pair} in the setting of group noise. The last line, \textit{i.e.}, Eq.~\eqref{eq:loss_max_matching}, is its lower bound, which alleviates the optimization obstacle caused by integral. Optimizing such a lower bound yields the optimization of the first line, and progressively makes the learning procedure approach the optimal mapping functions $f^*$ and $g^*$.
\end{proof}

Due to the adverse impact of group noise, $\hat{P}(y_i|\bar{x}_{ik};f,g)$ may still memorize some pairwise relation noise. In this case, the second term $\hat{P}(\bar{x}_{ik}|\bar{X}_i;f,g)$ can leverage the non-IID characteristics of the objects in the group to sufficiently capture their correlation, and distinguish the irregular noise by measuring their importance regarding the group. Besides, the max-pooling operation in Eq.~\eqref{eq:loss_max_matching} guarantees that only the most confident object is used, reducing the risk of group noise as much as possible.

\subsection{Implementation}
\begin{figure}
    \centering
    \includegraphics[width=2.8in]{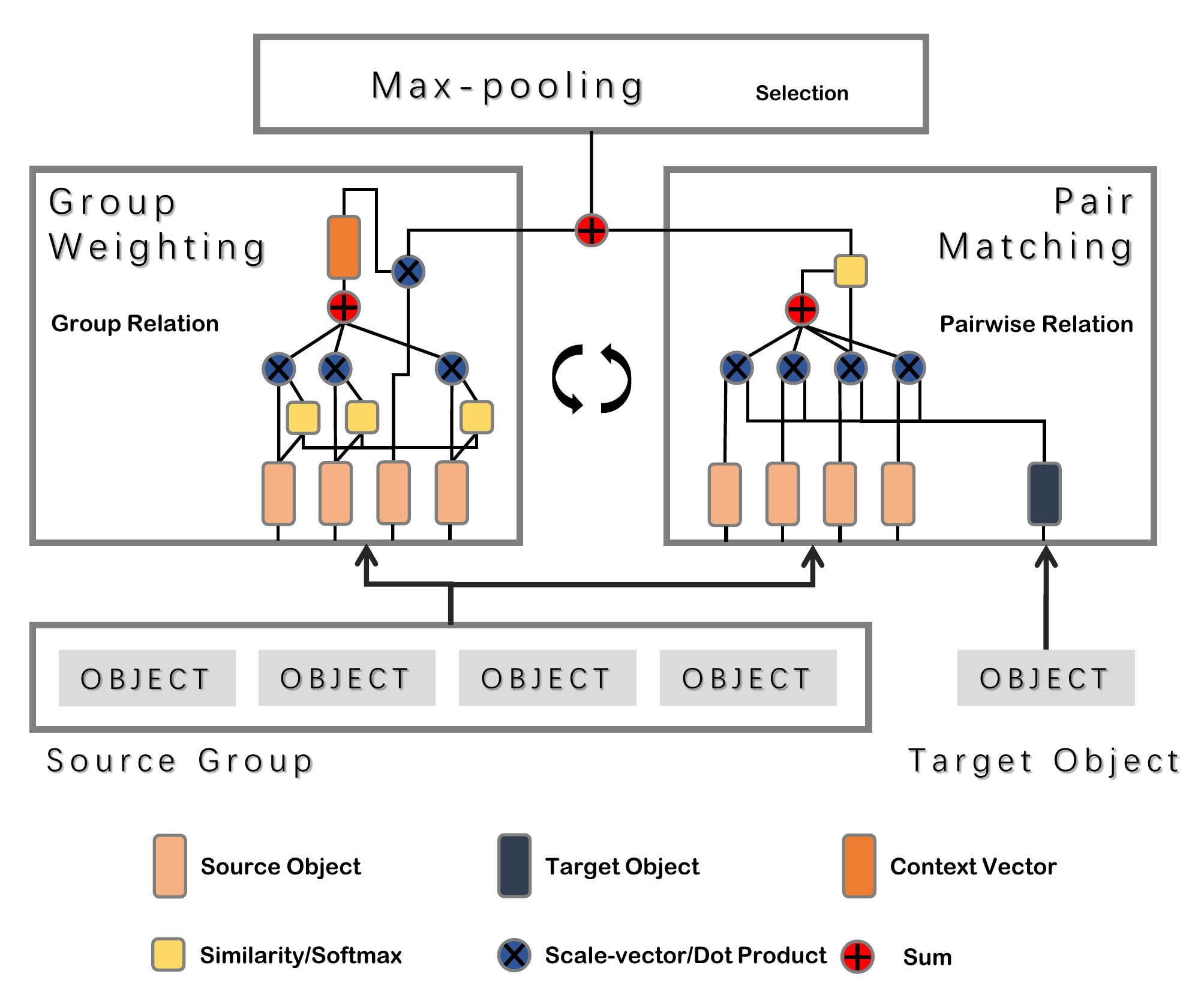}
    \caption{Max-Matching. The \emph{pair matching} evaluates the confidence of individual connections between source objects and the target. The \emph{group weighting} captures the object correlations in the group by measuring their importance. They are combined to form the final matching scores, followed by a max-pooling selection to choose the trustworthy object regarding the target. Group weighting and pair matching are interdependent and benefit from each other during training. }
    \label{fig:mm}
\end{figure}

In this section, we give our implementation of Eq.~\eqref{eq:loss_max_matching} in detail. First, the term $\hat{P}\left(y_j|\bar{x}_{ik};f,g\right)$ is named as the \emph{pair matching}, as it is the probability of matching between the source object $\bar{x}_{ik}$ and the target object $y_i$. 
It is estimated by the Softmax on the inner product of their embedding vectors, constructed as follows: 
\begin{equation}
    \hat{P}(y_i|\bar{x}_{ik};f,g) = \frac{\exp\{f(\bar{x}_{ik})^\top g(y_i)\}}{\sum_{y\in \mathcal{Y}}\exp\{f(\bar{x}_{ik})^\top g(y)\}}. \label{eq:softmax}
\end{equation}
The second term $\hat{P}(\bar{x}_{ik}|\bar{X}_i;f,g)$ aims to capture the object correlation by measuring the importance of $\bar{x}_{ik}$ regarding the group. It is termed as a \emph{group weighting} mechanism, as it assigns different weights for pair matching regarding the group. 
Accordingly, the Non-IID property in the group is considered, since the group weighting is essentially designed as a cluster-aware weighting method. Note that, the weights can either be calculated based on the embedding features $f(\bar{x})$ or the probabilistic features $\hat{P}(y|\bar{x};f,g)$.
To unify these two operations together, we denote the mapping function $h(\cdot)$ for the input features of the group weighting with the similarity measurement $\mathcal{S}(\cdot,\cdot)$. Then, the group weighting $\hat{P}(\bar{x}_{ik}|\bar{X}_i;f,g)$ is calculated by following steps: 
\begin{itemize}
    \item a) Measuring the similarity of the object $\bar{x}_{ik}$ with all other objects in the group (denoted by $\bar{x}'_{i1},\ldots,\bar{x}'_{i,K-1}$):
\begin{align}
    s_{\bar{x}_{ik}} =\left[\mathcal{S}(h(\bar{x}_{ik}),h(\bar{x}'_{i1})),\ldots,\mathcal{S}(h(\bar{x}_{ik}),h(\bar{x}'_{i,K-1}))\right]^\top,
    \nonumber
\end{align}
    and normalizing by Softmax $\tilde{s}_{\bar{x}_{ik}} = \operatorname{Softmax}(s_{\bar{x}_{ik}})$;
    \item b) Calculating the final weight of the object $\bar{x}_{ik}$ in the group with Sigmoid:
    \begin{align}
        \hat{P}(\bar{x}_{ik}|\bar{X}_i;f) = \operatorname{Sigmoid}\left(\mathcal{S}({c}_{\bar{x}_{ik}},h(\bar{x}_{ik}))\right),
        \label{eq:selfattention}
    \end{align}
    where ${c}_{\bar{x}_{ik}}$ is the context vector w.r.t. the object $\bar{x}_{ik}$, calculated by ${c}_{\bar{x}_{ik}} = \sum_{l=1}^{K-1}\tilde{s}_{\bar{x}_{ik},l}h(\bar{x}'_{il})$.
\end{itemize} 
The context vector $c_{\bar{x}_{ik}}$ is constructed by the weighted sum of all the other objects in the group, in which the weights $\tilde{s}_{\bar{x}_{ik}}$ assign the higher values for those objects similar to $\bar{x}_{ik}$. Intuitively, the context vector resembles the original $\bar{x}_{ik}$ if there exists plenty of objects in the group that are similar to the object $\bar{x}_{ik}$. A large value of group weighting (or a large $\mathcal{S}(c_{\bar{x}_{ik}},h(\bar{x}_{ik}))$) indicates that the object $\bar{x}_{ik}$ deserves more attention regarding its owning group $\bar{X}_i$.

By mixing the pair matching and the group weighting, we have the final matching score that evaluates the object confidence regarding the target as well as the group. A large value of the matching score generally indicates the corresponding object is trustworthy in its fine-grained relation to the target. The selection procedure is then deployed upon the matching mechanism via a simple max-pooling operation. It selects the object that is the most confident in terms of the pairwise relationship, and the irrelevant objects can be dropped. The model structure is summarized in Figure~\ref{fig:mm}.

\section{Experiments}

\subsection{Experimental Settings}

\begin{table}[t]
    \centering
    \small
    
    \begin{tabular}{c|cc|cc|cc}
        \toprule
         & \multicolumn{2}{c|}{Object} & \multicolumn{2}{c|}{Function} & \multicolumn{2}{c}{Weighting} \\
         \cline{2-7}
          & $\mathcal{X}$ & $\mathcal{Y}$ &  $f(\cdot)$  & $g(\cdot)$ & $h(\cdot)$ & $\mathcal{S}(\cdot,\cdot)$ \\
         \hline
         MIL & ins & lab   & ide & emb & Eq.~\eqref{eq:softmax} & neg-KL\\
         PLL & lab  & ins  & emb & lin & $f(\cdot)$ & dot \\
         RS  & item & item & emb & emb & $f(\cdot)$ & dot   \\
         \toprule
    \end{tabular}
    \caption{The Specification of Max-Matching on three types of learning settings with group noise.}
    \label{tab:max-matching settings}
\end{table}

To demonstrate the effectiveness of Max-Matching, we conduct extensive experiments in three representative learning settings with group noise, including MIL, PLL, and RS.  Table~\ref{tab:max-matching settings} summarizes their specifications regarding sample sets (\textit{i.e.,} $\mathcal{X}, \mathcal{Y}$), mapping functions (\textit{i.e.,} $f,g$), and group weighting (\textit{i.e.,} $\mathcal{S},h$). Therein, ``ins'', ``lab'', and ``item'' respectively denotes the instance with features, the label, and the item ID. Moreover, ``emb'' represents the embedding function that maps discrete category labels or item IDs to the embedding space; ``lin'' is a linear function for the instances with normalized features; and ``ide'' is the identity function.

\begin{figure*}[!t]
    \centering
    \includegraphics[width=7in]{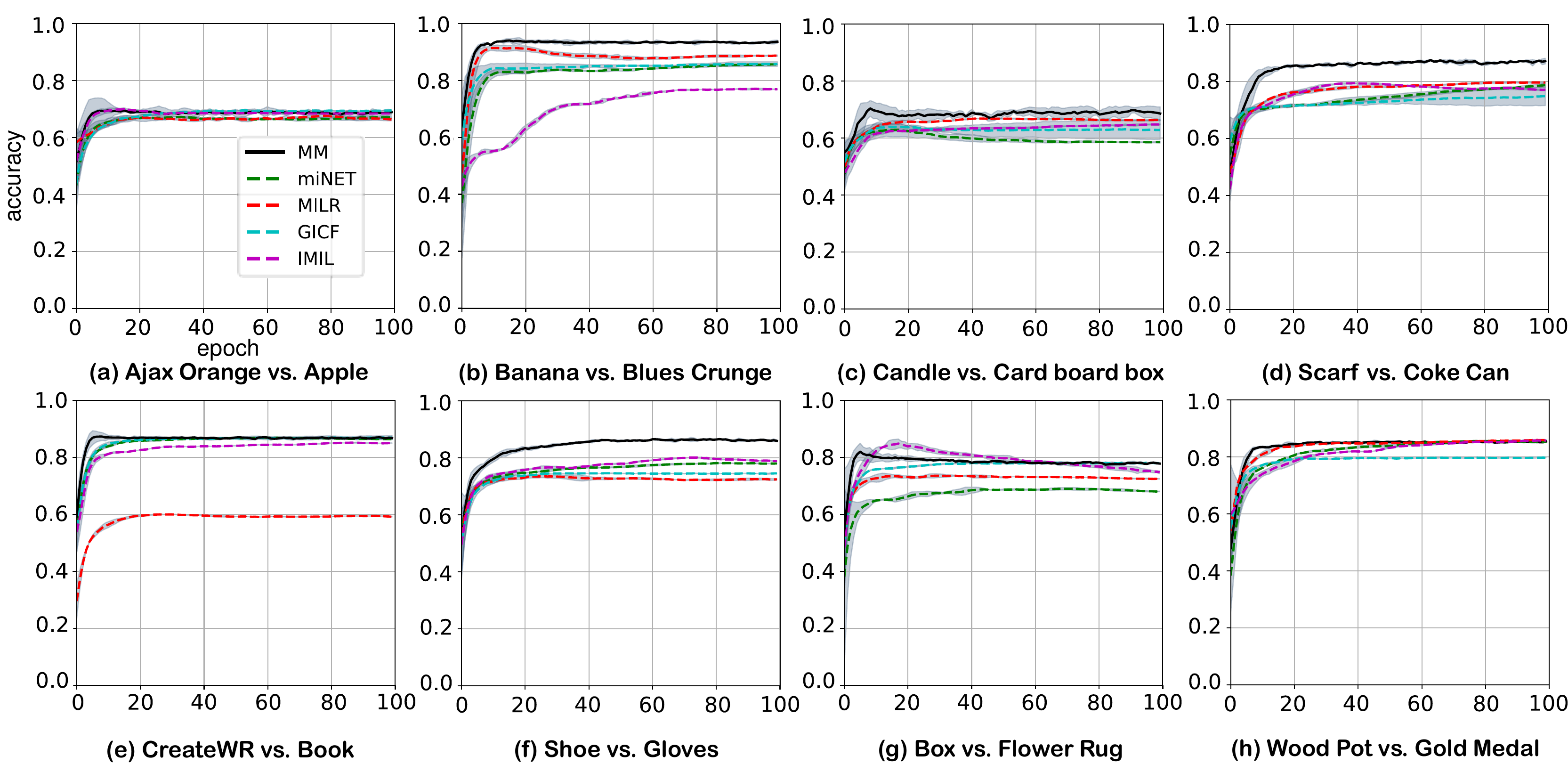}
    \vspace{-11pt}
    \caption{The test accuracy curves on \emph{SIVAL} for learning with group noise. Colored curves show the mean accuracy of 5 trials and shaded bars denote  standard deviation.}
    \label{fig:sival}
\end{figure*}

Our specification in MIL degenerates Eq.~\eqref{eq:softmax} into a linear function with Softmax, and its outputs are the inputs of group weighting with negative KL-divergence (neg-KL) as the similarity metric. By contrast, in PLL, instances and labels are both mapped, where Max-Matching can explore the non-IID characteristics of labels in the embedding space, and dot product (dot) is adopted as a proper metric. Similar deliberation holds for RS to measure the confidence of matching in the embedding space.  Moreover, we implement Max-Matching using PyTorch, the Adam~\cite{kingma2014adam} is adopted with the learning rate selected from $\{10^{-1}, \cdots, 10^{-4}\}$, and the methods are run for 50 epochs.

\subsection{Application to Multiple-Instance Learning}

In this section, we focus on the MIL setting, where we aim to learn an instance classifier given instances with only bag labels.  Here, instances in the bag that may deviate from their bag labels introduce group noise. 

The experiments are conducted on an object localization dataset \emph{SIVAL}~\cite{rahmani2005localized} in the literature of MIL, as it provides instance-level annotations for evaluation. We compare Max-Matching with two state-of-the-arts that focus on the instance classification, IMIL~\cite{peng2019address} and GICF~\cite{kotzias2015group}; two strong baselines that estimate instance labels in an intermediate step for bag classification, MILR~\cite{ray2005supervised} and miNET~\cite{wang2018revisiting}. Since the baselines only focus on binary classification, we use the data of each adjacent classes to construct the binary classification datasets. Each dataset is then partitioned into 8:1:1 for training, validation, and test. 

The experimental curves in terms of the test accuracy are 
illustrated in Figure~\ref{fig:sival} with 5 individual trials. From them, we find Max-Matching achieves superior performance over the baselines in most cases. For two bag-level prediction methods, the test accuracy is not very competitive since they implicitly consider the instance labels in the bag. As suggested by~\cite{carbonneau2018multiple}, the instance-level performance cannot be guaranteed for MIL methods that only focus on the coarse-grained bag labels. For two instance-level methods, although they generally show better performance than MILR and miNET, they are still inferior to Max-Matching, since they fail to sufficiently leverage the correlation among objects in the group. The results demonstrate the effectiveness of our method in learning with group noise.

\begin{table}[t] 
    \centering
    \small
    
    \begin{tabular}{c|c|c|c}
    \toprule
        & \tabincell{c}{Accuracy} & \tabincell{c}{Selection}  & \tabincell{c}{non-IID} \\
        \hline
        Pairwise   & 0.302$\pm$0.005 & $\times$ & $\times$      \\
        Matching   & 0.324$\pm$0.002 & $\times$ & $\checkmark$   \\
        Maximizing & 0.315$\pm$0.006 & $\checkmark$ & $\times$ \\
        \hline
        Max-Matching &  \textbf{0.368$\pm$0.005} & $\checkmark$ & $\checkmark$ \\
    \bottomrule
    \end{tabular}
    \caption{Average test accuracy and standard deviation in learning with group noise on \emph{SIVAL}.}
    \label{tab:sival}
\end{table}

\begin{figure*}[!t]
    \centering
    \includegraphics[width=7in]{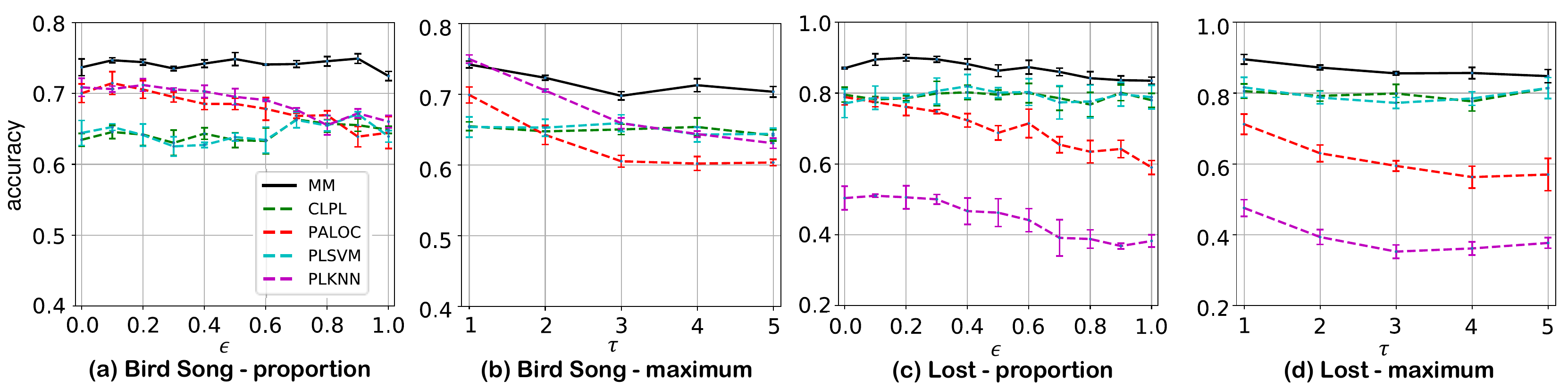}
    \vspace{-25pt}
    \caption{Test accuracy curves on PLL datasets for learning with group noise. Colored curves show the mean accuracy of 5 trials and error bars denote standard deviation.
    Therein, $\epsilon$ denotes the proportion of instances that are partially labeled, and $\tau$ is the maximum numbers of partial labels. }
    \label{fig:pll_acc} 
\end{figure*}

Furthermore, we conduct multi-class classification experiments on \emph{SIVAL}, since our method is not restricted to the binary classification. To show the advantages of the matching mechanism and the selection procedure in Max-Matching, we leverage following three baselines for the ablation study: 
\begin{itemize}
    \item \textbf{Pairwise}: Taking the group label $y_i$ as the label for each instance $\bar{x}_{ik}$ in the group $\bar{X}_i$, the objective can be written as
    $\sum_{i=1}^{n}\sum_{k=1}^{K} -\log \hat{P}(y_i|\bar{x}_{ik};f,g)$.
    \item \textbf{Matching}: Taking the matching scores of individuals in the group as the objective directly, which means
    $\sum_{i=1}^{n}\sum_{k=1}^{K}-\log \hat{P}(y_i|\bar{x}_{ik};f,g)-\log \hat{P}(\bar{x}_{ik}|\bar{X}_i;f)$.
    \item \textbf{Maximizing}: Selecting the most confident instance $\bar{x}_{ik}$ only in terms of the pairwise relation to the target, namely,
    $\sum_{i=1}^{n}-\max_{\bar{x}_{ik}\in \bar{X}_i} \log \hat{P}(y_i|\bar{x}_{ik};f,g)$.
\end{itemize}
The test accuracy with 5 individual trials for Max-Matching and three baselines are summarized in Table~\ref{tab:sival}. Accordingly, ``Pairwise'' achieves the worst test accuracy, sine the model directly fits the group noise and the Non-IID property of the group is simply ignored. Plugging the selection mechanism (``Maximizing") can generally perform better, and the similar result occurs in ``Matching'' that explores the non-IID property in the group. In comparison, Max-Matching, which both considers the object correlation and the pairwise relation, can significantly outperform all these baselines. Actually, we also find that tuning the trade-off between Maximizing and Matching can achieve further improvement. Therefore, it is possible to acquire a better performance to select a proper weight for two terms in Max-Matching.

\subsection{Application to Partial Label Learning}

\begin{table}
    \centering
    
    \small
    \vspace{3pt}
    \begin{tabular}{c||c|c|c|c|c}
        \toprule
        & \tabincell{c}{\emph{FG-}\\\emph{NET}} & \tabincell{c}{\emph{MSRC}\\\emph{v2}} & \tabincell{c}{\emph{Bird}\\\emph{Song}} &  \tabincell{c}{\emph{Yahoo!}\\\emph{News}} & \emph{Lost} \\
        \hline\hline
        PLKNN & \tabincell{c}{0.059$\pm$\\0.005} & \tabincell{c}{0.446$\pm$\\0.007} & \tabincell{c}{0.613$\pm$\\0.004} & \tabincell{c}{0.426$\pm$\\0.004} & \tabincell{c}{0.428$\pm$\\0.003}\\
        \hline
        PLSVM & \tabincell{c}{0.064$\pm$\\0.013} & \tabincell{c}{0.475$\pm$\\0.008} & \tabincell{c}{0.625$\pm$\\0.019} & \tabincell{c}{0.629$\pm$\\0.012} &  \tabincell{c}{0.801$\pm$\\0.025} \\
        \hline
        CLPL & \tabincell{c}{0.065$\pm$\\0.029} & \tabincell{c}{0.480$\pm$\\0.015} & \tabincell{c}{0.628$\pm$\\0.012} & \tabincell{c}{0.537$\pm$\\0.017} & \tabincell{c}{0.793$\pm$\\0.022} \\
        \hline
        PALOC & \tabincell{c}{0.054$\pm$\\0.005} & \tabincell{c}{0.463$\pm$\\0.011} & \tabincell{c}{0.598$\pm$\\0.020} & \tabincell{c}{0.434$\pm$\\0.0013} & \tabincell{c}{0.546$\pm$\\0.007}\\
        \hline
        \tabincell{c}{Max-\\Matching}             & \textbf{\tabincell{c}{0.110$\pm$\\0.021}} &
        \textbf{\tabincell{c}{0.517$\pm$\\0.007}} & \textbf{\tabincell{c}{0.642$\pm$\\0.010}} &
        \textbf{\tabincell{c}{0.647$\pm$\\0.005}} & \textbf{\tabincell{c}{0.823$\pm$\\0.025}}\\
        \toprule
    \end{tabular}
    \caption{The average test accuracy and its standard deviation on the PLL datasets in learning with group noise.}
    \label{tab:pll_testacc}
\end{table}

In this section, we validate Max-Matching in the setting of PLL, in which each instance is assigned with a set of candidate labels and only one of them is correct. 

The experiments are conducted on five PLL datasets from various domains: \emph{FG-NET}~\cite{panis2014overview} aims at facial age estimation; \emph{MSRCv2}~\cite{liu2012conditional} and \emph{Bird Song}~\cite{briggs2012rank} focus on object classification; \emph{Yahoo! News}~\cite{guillaumin2010multiple} and \emph{Lost}~\cite{cour2011learning} deal with face naming tasks. Each dataset is partitioned randomly into 8:1:1 for training, validation, and test. We compare Max-Matching with four popular PLL methods, including a non-parametric learning approach PLKNN~\cite{hullermeier2006learning}; a maximum margin based method PLSVM~\cite{nguyen2008classification}; a statistical consistent method CLPL~\cite{cour2011learning}; and a decomposition based approach PALOC~\cite{wu2018towards}.

The test accuracy of 5 individual trials for our method and baselines are reported in Table~\ref{tab:pll_testacc}. According to the results, PALOC shows extremely poor performance on datasets like \emph{Bird Song} and \emph{Lost}. This is because  it has no explicit denoising mechanism to avoid the influence of group noise. PLKNN also achieves relatively inferior results due to its strong assumption on the data distribution. Although PLSVM and CLPL can generally perform better, they still fail to explore the non-IID characteristics of candidate labels. In comparison, Max-Matching have the best performance among all these methods, as it further considers the correlations among the candidate labels. Notably, on \emph{FG-NET}, a challenging PLL dataset with a great many of strongly correlated candidate labels (7.48 partial labels per instance on average), Max-Matching is 4.37\% better than the second best method CLPL on average. 

To study the robustness of these methods in learning with different levels of group noise, we further conduct experiments on \emph{Lost} and \emph{Bird Song} with controlled proportion $\epsilon$ of partial labeled instances and controlled maximum numbers $\tau$ of partial labels. 
The test accuracy for varying $\epsilon$ and $\tau$ is summarized in Figure~\ref{fig:pll_acc}. Similar to the above results, PLKNN is unstable across these two datasets due to its assumption on data distribution. PALOC is also vulnerable to the group noise, and its accuracy drops quickly with the growth of $\epsilon$ and $\tau$. Although the performances are relatively stable for CLPL and PLSVM, their test accuracy is consistently inferior to Max-Matching. These results further demonstrate the effectiveness of Max-Matching in PLL.    

\subsection{Application to Recommender System} 
Finally, we conduct experiments of recommendation, which aims at recommending points of interest to the users, \textit{e.g.}, item recommendation in e-commerce. The classical item-based collaborative filtering~\cite{sarwar2001item} critically depends on the trustworthy pairwise relationship, which is not practical on e-commercial websites. Generally, due to the varying interests of the user, his/her historically visited items are not always relevant to the subsequent items. Then, taking the user click sequence as a group and the next item as the target, we have the coarse relation as Figure~\ref{fig:motivation}(d). As a result, we face the problem of learning with group noise when applying the item-based collaborative filtering. 

The offline experiments are implemented on a range of datasets from Amazon: \emph{Video}, \emph{Beauty}, and \emph{Game}. In each dataset, the visited items of each user are segmented into subsets with at most 6 items, where the last item of each subset is taken as the target, and the others are taken as the group with noise. For each user, we randomly take two subsets for validation and test, and the remaining data are used for training. In the experiments, we consider several classical and advanced baselines, including a simple method that ranks items according to their popularity and recommends new items regarding the co-occurrence, Pop; a popular collaborative filtering method, I-CF~\cite{linden2003amazon}; and two deep model based approaches that exploit the sequential behavior in the group, Caser~\cite{tang2018personalized} and Att~\cite{zhang2019next}. For Max-Matching, we recommend new items by ranking the probabilities $\hat{P}\left(y|x;f,~g\right)$, where $x$ can be the second last visited item (MM), or any item in the considered group (MM+)\footnote{Note that, MM+ is to compare the sequence-based recommendation methods Caser and Att which use all items in the group.}. Following~\cite{zhang2019next}, we report the performance on two widely used metrics, HIT@10 and NDCG@10. HIT@10 counts the fraction of times that the true next item is in the top-10 items, while NDCG@10 further assigns weights to the rank.

The average results of 5 individual trials in terms of the HIT@10 and NDCG@10 are summarized in Table~\ref{tab:amazon_results}. First, we compare MM with Pop and I-CF, which all recommend new items according to the last visited ones of users. Pop always shows extremely poor performance, as it is based on the popularity and cannot learn the correlations between items. While I-CF performs much better, it relies on the reliable pairwise relations without considering the group noise. By contrast, MM is robust to the fine-grained uncertain relations, which achieves the significant improvements. 
Second, we compare MM+ with Caser and Att, which are two recommendation approaches that can implicitly model the group relation. However, they mainly focus on the temporal behavior of users, making them fail to explicitly distinguish true relations from the irrelevant noise. By contrast, MM+ considers both the group relation and the denoising mechanism, and the experimental results on average demonstrate its effectiveness and rationality.

\begin{table}[t]
    \centering
    \small
    
    \begin{tabular}{c||c|c|c|c|c|c}
        \toprule
        & \multicolumn{2}{c|}{\emph{Video}} & \multicolumn{2}{c|}{\emph{Beauty}} & \multicolumn{2}{c}{\emph{Game}} \\
        \hline\hline
         & HIT & NDCG & HIT & NDCG & HIT & NDCG \\
        \hline
        Pop & \tabincell{c}{0.515} & \tabincell{c}{0.397} & \tabincell{c}{0.401} & \tabincell{c}{0.258} & \tabincell{c}{0.402} & \tabincell{c}{0.252} \\
        \hline
        I-CF & \tabincell{c}{0.622} & \tabincell{c}{0.420} & \tabincell{c}{0.429} & \tabincell{c}{0.285} & \tabincell{c}{0.405} & \tabincell{c}{0.298} \\
        \hline
        MM & \textbf{\tabincell{c}{0.692}} & \textbf{\tabincell{c}{0.471}} & \textbf{\tabincell{c}{0.543}} & \textbf{\tabincell{c}{0.381}} & \textbf{\tabincell{c}{0.495}} & \textbf{\tabincell{c}{0.332}} \\
        \hline\hline
        Caser & \tabincell{c}{0.643} & \tabincell{c}{0.425} & \tabincell{c}{0.523} & \tabincell{c}{0.345} & \tabincell{c}{0.493} & \tabincell{c}{0.311} \\
        \hline
        Att & \tabincell{c}{0.624} & \tabincell{c}{0.429} & \tabincell{c}{0.445} & \tabincell{c}{0.335} & \tabincell{c}{0.427} & \tabincell{c}{0.292} \\
        \hline
        MM+ & \textbf{\tabincell{c}{0.694}} & \textbf{\tabincell{c}{0.473}} & \textbf{\tabincell{c}{0.561}} & \textbf{\tabincell{c}{0.389}} & \textbf{\tabincell{c}{0.518}} & \textbf{\tabincell{c}{0.345}}\\
        \toprule
    \end{tabular}
    \caption{Average HIT@10 (HIT for short) and NDCG@10 (NDCG for short) with standard deviation on Amazon.}  
    \label{tab:amazon_results}
\end{table}

We also conduct online experiments by deploying Max-Matching to the recommender system on one e-commerce platform. Like many large-scale recommenders, it consists of two stages, the recall stage and the ranking stage. The recall stage generates the most relevant candidate items that are related to the visited items of users in the middle-scale. The ranking stage scores the candidates in a fine-grained granularity for the top-k recommendation. We deploy Max-Matching to the recall stage and compare our method with the online item-based NeuralCF~\cite{he2017neural}. Here, NeuralCF is supervised by the pairwise relations manually extracted from the user-click sequence. After one-week experiments, we achieved about $10\%$ improvement on click-through rate (CTR). Figure~\ref{fig:rs_ncf_mm} illustrates one example of the top-5 recommendation from NeuralCF and Max-Matching. According to the results, we can find the dress recommendation from NeuralCF is mixed with the shorts, which in fact, origins from the training with non-ideal pairwise relations. 
 
\begin{figure}
    \centering
    \includegraphics[width=3.in]{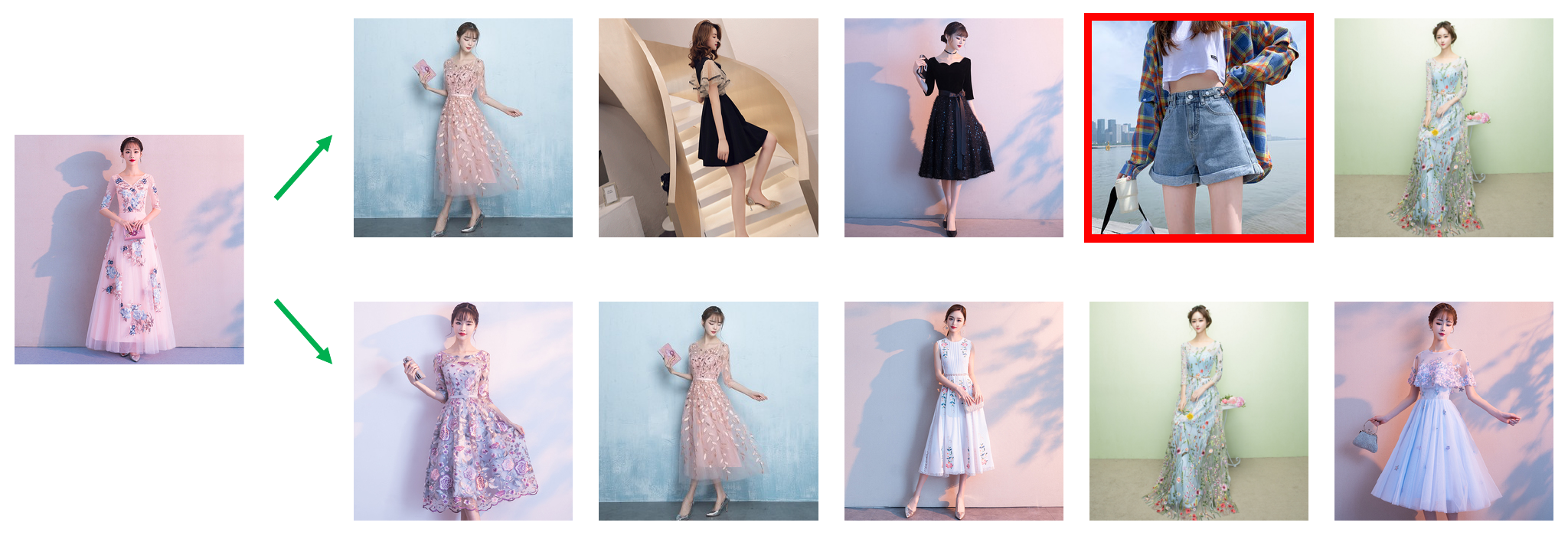}
    \vspace{-7pt}
    \caption{The top-5 recall examples from NeuralCF and Max-Matching given the user clicked dress items. The 5 candidates in the first row are recalled by NeuralCF and the 5 candidates in the second row are recalled by Max-Matching.}
    \label{fig:rs_ncf_mm}
\end{figure}
 
\section{Conclusion}
In this paper, we focus on the learning paradigm with group noise, where a group of correlated objects are collectively related to the target with fine-grained uncertainty. To handle the group noise, we propose a novel Max-Matching mechanism in selecting the most confident objects in the group for training, which considers both the correlation among the group as well as the pairwise matching to the target. The experimental results in three different learning settings demonstrate its effectiveness. In the future, we will generalize Max-Matching to handle the independent pairwise relations, e.g., learning with label noise, and explore a better trade-off between two terms in our objective.

\section{Acknowledgments}
    JCY and HXY was supported by NSFC No. U20A20222. CG was supported by NSFC No. 61973162, the Fundamental Research Funds for the Central Universities No. 30920032202, CCF-Tencent Open Fund No. RAGR20200101, the “Young Elite Scientists Sponsorship Program” by CAST No. 2018QNRC001, and Hong Kong Scholars Program No. XJ2019036. TLL was supported by Australian Research Council Project DE-190101473. BH was supported by the RGC Early Career Scheme No. 22200720, NSFC Young Scientists Fund No. 62006202, HKBU Tier-1 Start-up Grant, HKBU CSD Start-up Grant, and HKBU CSD Departmental Incentive Grant. 
    
\bibliography{main}
\end{document}